\documentclass[11pt]{article}

\RequirePackage{amsthm,amsmath,amsfonts,amssymb,bbm}
\RequirePackage{mathtools}
\RequirePackage{microtype}
\RequirePackage{graphicx}
\RequirePackage{subcaption}
\RequirePackage{booktabs}
\RequirePackage{array}
\RequirePackage{comment}
\RequirePackage{float}
\RequirePackage[numbers]{natbib}
\RequirePackage[colorlinks,citecolor=blue,urlcolor=blue]{hyperref}
\RequirePackage[capitalize,noabbrev]{cleveref}

\numberwithin{equation}{section}

\theoremstyle{plain}

\newtheorem{theorem}{Theorem}[section]

\theoremstyle{definition}

\newtheorem{assumption}[theorem]{Assumption}

\theoremstyle{remark}

\title{Conformalized Percentile Interval: Finite Sample Validity and Improved Conditional Performance}
\author{
Ran Zou\thanks{Department of Statistics, University of California, Irvine}
\and
Wanrong Zhu\footnotemark[1]
\and
Bin Nan\footnotemark[1]
}
\date{}

\begin{document} 
	\maketitle

\begin{abstract}
 Conformal prediction provides distribution-free predictive intervals with finite-sample marginal coverage. However, achieving conditional validity and interval efficiency (in terms of short interval length) remains challenging, particularly in complex settings with heteroskedasticity, skewed responses, or estimation errors. We propose a method that leverages the conditional cumulative distribution function (CDF) estimated via neural networks for a conformal-style calibration on responses transformed via the probability integral transform (PIT) to construct a finite-sample-adjusted percentile interval with the shortest length determined by the estimated conditional CDF. Calibrating in PIT space is effective because PIT values are asymptotically feature-independent when the CDF estimator is accurate, which mitigates feature-dependent miscoverage and improves conditional calibration. On the other hand, our percentile calibration adapts to the empirical PIT distribution, which is robust against a possibly imperfect estimation of the conditional CDF.
 We prove the finite-sample marginal coverage property of the proposed method and show its asymptotic conditional coverage under mild consistency conditions. Experiments on diverse synthetic and real-world benchmarks demonstrate better conditional calibration and substantially shorter intervals than existing methods.
\end{abstract}

\section{Introduction}
\label{sec:intro}
Reliable uncertainty quantification (UQ) is increasingly important as machine learning systems are deployed in
high-stakes settings. While many classical UQ methods provide asymptotic guarantees or rely on strong distributional assumptions \citep{koenker2001quantile, ruppert2003semiparametric}, real-world applications demand finite-sample validity that remains reliable under complex data and model misspecification. 
Conformal prediction provides a compelling solution: it can be wrapped around any predictive model (including modern deep networks and ensemble methods) and, without assumptions beyond data exchangeability, yields prediction sets with finite-sample marginal coverage
\citep{vovk2005algorithmic,lei2018distributionfree}. Specifically, given any pre-trained predictive model and a calibration dataset $\{(X_{1}, Y_{1}), \dots, (X_{n}, Y_{n})\}$, where $X_i$'s are covariates and $Y_i$'s are corresponding responses, conformal prediction learns a prediction-set function $\widehat{C}(\cdot)$ with a user-specified target coverage level $1-\alpha$, where $\alpha$ is typically a small value, e.g., $\alpha=0.1$. For a new test point with covariate $X_{n+1}$, it constructs a prediction set $\widehat{C}(X_{n+1})$ for the corresponding response $Y_{n+1}$ such that 
\[\mathbb{P}(Y_{n+1}\in\widehat{C}(X_{n+1}))\ge 1-\alpha.\]

The marginal coverage guarantee provided by standard conformal prediction holds only on average over all possible test points. It may be insufficient for many high-stakes applications that require stronger reliability guarantees conditioned on specific test instances or subpopulations.  For such applications, one may seek prediction sets that satisfy \emph{test-conditional} coverage:
\[\mathbb{P}(Y_{n+1}\in\widehat{C}(X_{n+1})\mid X_{n+1})\ge 1-\alpha.\] 
However, it has been shown that achieving exact test-conditional coverage while maintaining informative prediction sets is impossible in a fully distribution-free setting \citep{foygel2021limits}.  For practical purposes, one may instead relax this requirement to hold in an asymptotic sense and aim to improve conditional performance as much as possible. This forms the central objective of this work.
 
Achieving test-conditional coverage is substantially more challenging than marginal coverage. For example, standard residual-based conformal intervals rely on a global residual distribution, which can lead to feature-dependent miscoverage when the residual law varies across the feature space.  
To mitigate this, several methods have focused on designing adaptive conformity scores. A simple approach is to use variance-rescaled residual scores \citep{lei2018distributionfree}, which adjust for heteroskedasticity by normalizing residuals with an estimate of local variability. Going beyond scale adjustment, Conformalized Quantile Regression (CQR) \citep{romano2019cqr} incorporates conditional quantile information to better capture heterogeneity in the response. To further incorporate richer distributional information, several distribution-aware methods have been proposed, including conditional histogram regression \citep{sesia2021conformal} and Distributional Conformal Prediction (DCP) \citep{chernozhukov2021dcp}. Among these, DCP is most closely related to our work, as it also leverages conditional CDF estimation and performs calibration in the Probability Integral Transform (PIT) space.
If the conditional CDF is estimated accurately, then the transformed variable should lie on the unit interval and be approximately uniformly distributed, with reduced dependence on the covariates. However, existing CDF-based approaches typically rely on a single conformity score in the PIT space, which may impose structural constraints (e.g., symmetry) and limit their ability to adapt to asymmetric errors.

Building on these ideas, we propose the Conformalized Percentile Interval (CPI). CPI performs endpoint calibration rather than score calibration, directly adjusting the lower and upper percentile bounds independently without any symmetry restriction. 
Specifically, using PIT values $U_i = \widehat F(Y_i \mid X_i)$ computed on a hold-out calibration set, we select lower and upper PIT cutoffs, $u_{lo}$ and $u_{hi}$, as independent order statistics of the calibration PIT values. For a new feature vector $X_{n+1}$, the adjusted percentile interval is constructed as $$[\widehat F^{-1}(u_{lo} \mid X_{n+1}), \widehat F^{-1}(u_{hi} \mid X_{n+1})].$$ 
This approach adapts to the local distributional shape while retaining finite-sample marginal validity. In principle, our framework is flexible and can be paired with any conditional CDF estimator. In this work, we adopt the neural network conditional distribution estimator of \citet{hu2023nncde} for our experiments to demonstrate the practical performance of our framework and provide a concrete, specific implementation.



\textbf{Our main contributions are:}
\begin{itemize}
    \item \textbf{A Novel PIT-Based Framework:} We move beyond classic conformity scores to propose an adjusted percentile interval that directly calibrates percentile endpoints. The method is straightforward to implement, provides finite-sample marginal guarantees, and is compatible with any conditional CDF estimator.
        \item \textbf{Improved Conditional Coverage and Shorter Intervals:} By leveraging conditional distribution estimation, CPI better captures local distributional shape, leading to improved and more stable conditional coverage. This improved alignment with the conditional distribution also results in shorter prediction intervals, particularly in settings with heteroskedasticity or asymmetric noise.
    \item \textbf{Robustness to Asymmetric Estimation Errors:} Unlike existing CDF-based methods, which rely on a single conformity score in the PIT space and often impose symmetry, CPI independently calibrates upper and lower PIT cutoffs. This flexibility makes the framework robust to asymmetric estimation errors---common in cases of model misspecification or distribution shift---resulting in substantially shorter intervals without sacrificing coverage.
\end{itemize}


 
\subsection{Related Work}
A separate line of work improves local adaptivity by modifying the calibration procedure itself, for example by reweighting calibration points using neighborhood information around the test point \citep{tibshirani2019covariateshift, guan2023localized, hore2023conformal,
plassier2025rectifying,PlassierProb25,  gibbs2025conformal}. These approaches are also useful for improving subgroup-conditional performance or handling covariate shift. 
Efficiency in terms of interval width is another important aspect of prediction sets. A better fitted model typically leads to smaller prediction errors and thus narrower intervals \citep{burnaev2014efficiency, stutz2021learning}. Efficiency is also closely related to conditional coverage: adaptive conformity scores, such as those in \citet{lei2018distributionfree, romano2019cqr}, generally produce smaller prediction sets by accounting for heterogeneity in the data. In addition, ensemble strategies that improve either the predictive model or the conformity score during pretraining or calibration have been shown to further enhance efficiency \citep{xie2024boosted, yang2025selection, liang2024conformal}. Finally, the Predictability, Computability, and Stability (PCS) framework for uncertainty quantification \citep{yu2013stability, pcsUQ2025} combines model selection or augmentation with bootstrap-style perturbations to achieve strong empirical efficiency.
These approaches are largely orthogonal to our framework and can potentially be combined with CPI to further improve conditional performance and efficiency.
\section{Conformalized Percentile Interval (CPI)}
\label{sec:method}
We consider the same data-splitting setup as in split conformal prediction, which consists of three independent parts: a training set $\mathcal D_{\mathrm{tr}}$ for CDF estimation; a calibration set with $n$ observations
$
\mathcal D_{\mathrm{cal}}=\{(X_j,Y_j)\}_{j=1}^{n}
$
for finite-sample correction; and a test point with covariate $X_{n+1}$. Our goal is to construct a prediction set $\widehat{C}(X_{n+1})$ for the response $Y_{n+1}$ with a target coverage level of $1-\alpha$, $\alpha\in (0, 1)$.

Let $F(\cdot \mid x)$ denote the true conditional CDF of $Y$ given $X = x$. A percentile interval of the form $$\big[F^{-1}(z \mid x), F^{-1}(z+1-\alpha \mid x)\big]$$ achieves the target coverage level $1-\alpha$ for any starting point $z \in [0,\alpha]$, provided the conditional CDF is known. In this section, we construct a conformalized percentile interval by first forming an adjusted interval $[u_{\mathrm{lo}}, u_{\mathrm{hi}}]$ in the PIT space—serving as a finite-sample correction to the nominal interval $[z, z+1-\alpha]$—using an estimated CDF and the calibration set. Finally, we map this adjusted interval back to the original response scale.

As mentioned earlier, the training step can be replaced by any pre-trained conditional distribution estimator. In the experiment, we focus specifically on the neural network conditional distribution estimator proposed by \citet{hu2023nncde}, which we refer to as the NN-CDE model, and train it on $\mathcal D_{\mathrm{tr}}$. We show that it performs well within our procedure, making the overall approach practical and easy to implement in real-world settings.  
We use $\widehat F(\cdot \mid x)$ to denote the resulting estimate of the true conditional CDF $F(\cdot \mid x)$, and $\widehat Q(u \mid x) = \widehat F^{-1}(u \mid x)$ to denote the corresponding conditional quantile function for $u \in (0,1)$.

\subsection{General Framework}
\label{sec:CPI_base}
In general, the starting point $z$ for the nominal interval $[z, z+1-\alpha]$ can be an arbitrary function, provided it is constructed independently of the calibration set.
For example, $z$ could be a fixed constant, e.g. $z = \alpha/2$, or a pre-trained function $z = z(x)$ derived solely from the training data that depends on the test covariate $x$.   We will discuss the length-optimal choice of $z$ later.
The calibration procedure is detailed below and illustrated in Figure \ref{fig:cpi-illustration}.

\begin{figure}[t]
  \centering  \includegraphics[width=\columnwidth]{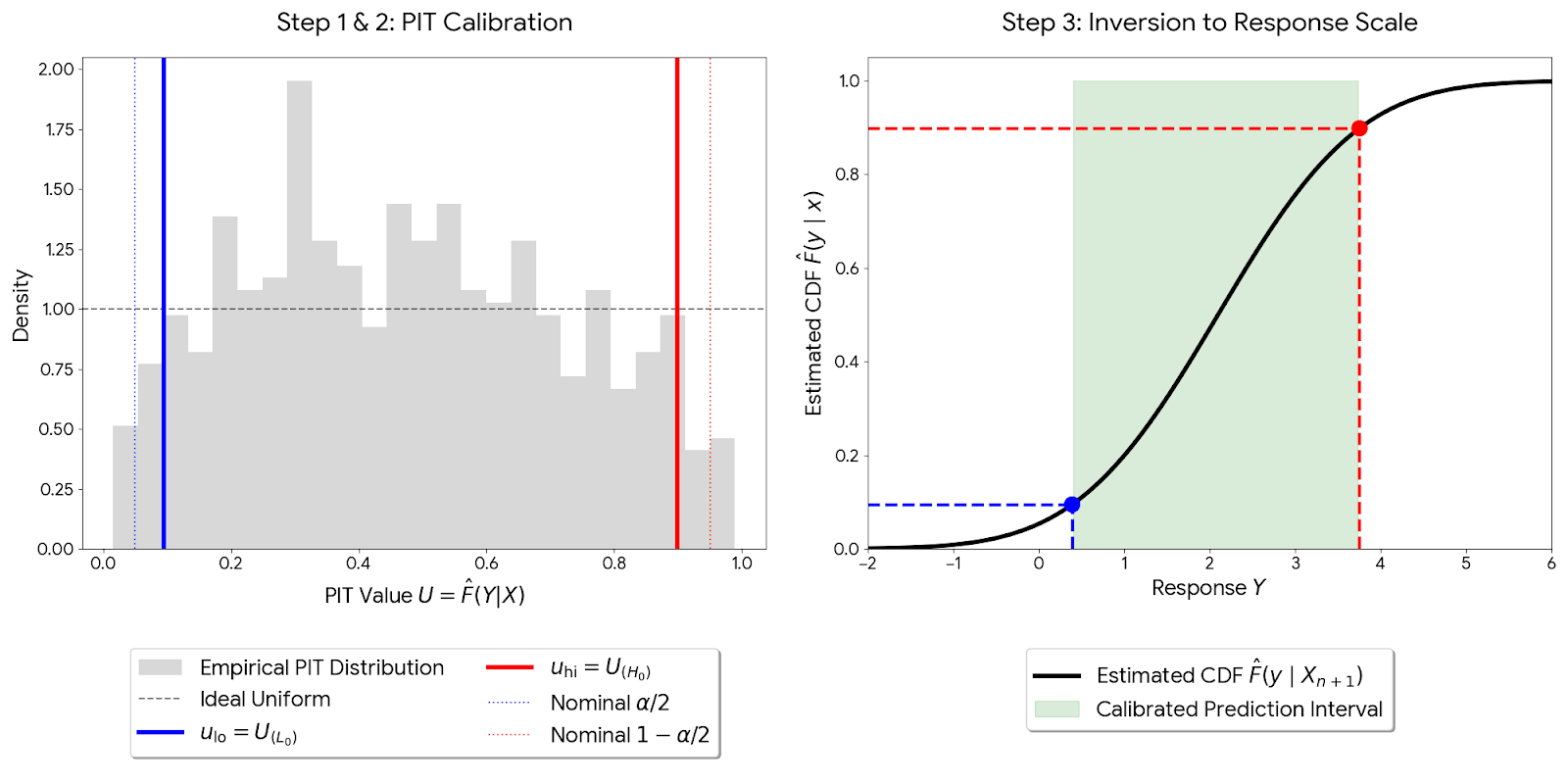}
  \caption{Illustration of CPI with $z = \alpha/2$. Given calibration data and a pre-trained 
conditional CDF estimator $\widehat{F}(\cdot \mid x)$, CPI (i) computes the 
calibration PIT values $U_j = \widehat{F}(Y_j \mid X_j)$, (ii) selects lower 
and upper PIT cutoffs $u_{\mathrm{lo}}, u_{\mathrm{hi}}$ independently as 
order statistics of $\{U_j\}$, and (iii) outputs the prediction interval 
$\big[\widehat{Q}(u_{\mathrm{lo}} \mid X_{n+1}),\, 
\widehat{Q}(u_{\mathrm{hi}} \mid X_{n+1})\big]$ by inverting $\widehat{F}$ 
at the calibrated cutoffs.}
  \label{fig:cpi-illustration}
\end{figure}

\emph{Step 1: Compute PIT values.} \\On the calibration set $\mathcal D_{\mathrm{cal}}$, compute the PIT values
\[
U_j=\widehat F(Y_j\mid X_j),\qquad j=1,\dots,n,
\]
and let $ U_{(1)}\le\cdots\le  U_{(n)}$ denote the corresponding order statistics (with randomized tie-breaking when necessary).

If the estimated CDF were equal to the true conditional CDF, then the PIT values $\widehat F(Y_{n+1} \mid X_{n+1})$ (as well as $U_j$) would follow a $\mathrm{Uniform}(0,1)$ distribution. In that case, the PIT interval $[z, z+ 1-\alpha]$ would be valid, and mapping it back via 
$\Big[
\widehat Q(z\mid x),\;
\widehat Q(z+1-\alpha\mid x)
\Big]$ would yield a percentile interval with conditional coverage $1-\alpha$. However, because $\widehat F$ is estimated, the PIT values are not exactly uniform. Following the idea of conformal prediction \citep{vovk2005algorithmic}, we therefore use the calibration set to better approximate the distribution of the PIT values.  

\emph{Step 2: Select lower and upper PIT cutoffs.}\\
Define the rank indices
\begin{align*}
L(z) &:= \big\lfloor z \,(n+1)\big\rfloor+1,\\
H(z) &:= \big\lceil (z +1-\alpha)\,(n+1)\big\rceil,
\end{align*}
and set the following PIT cutoff values from $\{U_{(1)}, \dots, U_{(n)}\}$:
\[
u_{\mathrm{lo}}:= U_{(L(z))},
\quad
u_{\mathrm{hi}}:= U_{(H(z))}.
\]
Then the adjusted PIT interval is $[u_{\mathrm{lo}}, u_{\mathrm{hi}}]$, which serves as a finite-sample correction to $[z, z+ 1-\alpha]$ from the empirical distribution of the PIT values.

\emph{Step 3: Invert the conditional CDF to form a prediction interval.}\\
The final prediction set for the response is a percentile interval based on the calibrated PIT cutoffs obtained in \emph{Step 2}. Specifically, given a new covariate value $X_{n+1}$, we output the prediction interval 
\[
\widehat C(X_{n+1})
=
\Big[
\widehat Q(u_{\mathrm{lo}}\mid X_{n+1}),\;
\widehat Q(u_{\mathrm{hi}}\mid X_{n+1})
\Big].
\]
Thus, the resulting interval can be interpreted as a percentile interval with a finite-sample adjustment to the nominal interval $\big[F^{-1}(z \mid x), F^{-1}(z+1-\alpha \mid x)\big]$. This adjustment corrects for the  estimation error by calibrating the estimated conditional CDF in PIT space, resulting in prediction intervals with reliable finite-sample marginal coverage.

 
\subsection{Length-optimal  Choice of $z$ }
\label{sec:z}
 CPI can be viewed as a percentile interval with a finite-sample adjustment applied to the nominal interval $\big[F^{-1}(z \mid x), F^{-1}(z+1-\alpha \mid x)\big]$. While fixing $z = \alpha/2$ serves as a solid basic method, from an efficiency perspective, we prefer intervals that capture the regions where the conditional distribution is most concentrated.  To achieve this, we now define a feature-dependent starting point, $z(x)$, explicitly designed to yield a more efficient---that is, the shortest possible---predictive interval.

Intuitively, this corresponds to selecting a PIT interval that aligns with the highest-density regions upon mapping back to the response scale. Ideally, we want to select the theoretical optimal starting point $z^\star(x)$ that minimizes the true interval length:$$z^\star(x) \in \arg\min_{z\in[0,\alpha]} \Big\{ F^{-1}\big(z+1-\alpha \mid x\big) - F^{-1}\big(z\mid x\big) \Big\}.$$Since the true conditional quantile function $F^{-1}$ is unknown in practice, we approximate this objective using our estimated conditional quantile function $\widehat{Q}$ (the inverse of our estimated CDF).  

This motivates modifying the calibration step by allowing the lower PIT cutoff to adapt to the test covariate $x$ rather than remaining fixed.  We propose two distinct approaches to handle this, each offering a different trade-off between offline training cost and online inference speed.

\emph{Approach 1 (Per-Test-Point Direct Optimization):} One can solve the minimization problem for $\hat{z}(x)$ on the fly for every incoming test point $x$:$$\hat{z}(x) \in \arg\min_{z \in [0, \alpha]} \left\{ \widehat{Q}(z + 1 - \alpha \mid x) - \widehat{Q}(z \mid x) \right\},$$
via a grid search over $[0, \alpha]$.
Then, when constructing the prediction interval for a new test covariate $X_{n+1}$, we set the starting point in \emph{Step 2} to $z = \hat{z}(X_{n+1})$.

\emph{Approach 2 (Amortized Optimization via a Learned Model):}  
During the training phase, we first compute the empirical optimal points $\hat{z}(x)$ as described in Approach 1 for all training covariates. Next, we fit a secondary neural network $z_\theta(x)$, using the training set $\mathcal{D}_{\mathrm{tr}}$, to predict $\hat z(x)$. To strictly enforce the boundary constraint $z_\theta(x) \in (0, \alpha)$, we parameterize the network output as:$$z_\theta(x) = \alpha \cdot \sigma\!\big(g_\theta(x)\big), \qquad \sigma(t) = \frac{1}{1+e^{-t}},$$where $g_\theta(x) \in \mathbb{R}$ is an unconstrained neural network. We train $g_\theta$ by minimizing the mean squared error over the training covariates:$$\min_\theta \; \frac{1}{|\mathcal{D}_{\mathrm{tr}}|} \sum_{(x,\cdot) \in \mathcal{D}_{\mathrm{tr}}} \Big(z_\theta(x) - \hat{z}(x)\Big)^2.$$Once training is complete, $z_\theta(\cdot)$ is frozen and treated as part of the base model. During calibration and testing, determining the optimal $z$ requires only a single, highly efficient forward pass. When constructing the prediction interval for a new test covariate $X_{n+1}$, we set the starting point in \emph{Step 2} to $z = z_\theta(X_{n+1})$.

Both approaches remain valid within the conformal framework because $\hat{z}(x)$ or $z_{\theta}(x)$ depend strictly on the training set and the test covariate $x$; they remain unconditionally independent of the calibration set.  Approach 1 may be preferred for its pipeline simplicity, as it circumvents the need to train and tune an additional model for $z_\theta$ and avoids an additional source of approximation error. However, it requires solving a grid-search minimization problem for every single test point. Consequently, for applications with extremely large test sets or strict latency requirements, Approach 2 becomes computationally superior by amortizing the heavy optimization cost into a single offline training phase. Users may choose between these two implementations based on their specific deployment constraints.

\section{Connection and comparison to DCP}
The same CDF estimator (e.g., an NN-CDE model) and the length-optimization strategy described in Section \ref{sec:z} can also be integrated into the DCP framework \cite{chernozhukov2021dcp}.
 While both CPI and DCP utilize an estimated conditional CDF to map responses into PIT space---thereby improving conditional coverage compared to non-CDF-based methods---the distinction between the two lies in how the PIT interval is calibrated (\emph{Step 2}).
 
 CPI calibrates the lower and upper PIT cutoffs \emph{independently} as order statistics of the calibration PITs. In contrast, DCP is a traditional conformal inference method that relies on a conformity score $S_j = |U_j - c|$, where $c$ is a fixed center of the PIT interval ($c = 0.5$ for basic DCP and $c = z+(1-\alpha)/2$ in the general case). DCP then forms the PIT cutoffs$$u_{\mathrm{lo}} = \max(0, c-\hat{q}), \qquad u_{\mathrm{hi}} = \min(1, c+\hat{q}),$$where $\hat{q}$ denotes the $(1-\alpha)$-quantile of $\{S_j\}_{j=1}^n$. The resulting PIT interval is \emph{symmetric} around the center $c$.
 
 In practice, the PIT distribution is often non-uniform and asymmetric due to finite-sample estimation error, model misspecification, or distribution shift. For instance, if the estimated CDF consistently overestimates the true CDF, the PIT values shift toward 1. The symmetry constraint causes DCP to produce unnecessarily wide intervals. The DCP center $c$ is confined to $[(1-\alpha)/2, \alpha + (1-\alpha)/2]$ since $z$ ranges over $[0, \alpha]$ (e.g., $[0.45, 0.55]$ for $\alpha = 0.1$). Even after optimizing $z$, the DCP center barely moves from 0.5. For a right-skewed PIT distribution whose mass is concentrated well above 0.5, this center is fundamentally misaligned, and the symmetric interval must extend far into low-density regions to achieve the required coverage. CPI removes this bottleneck; because its cutoffs are independent order statistics, the interval adapts naturally to the actual PIT distribution.

\begin{figure}[t]
  \centering
\includegraphics[width=\columnwidth]{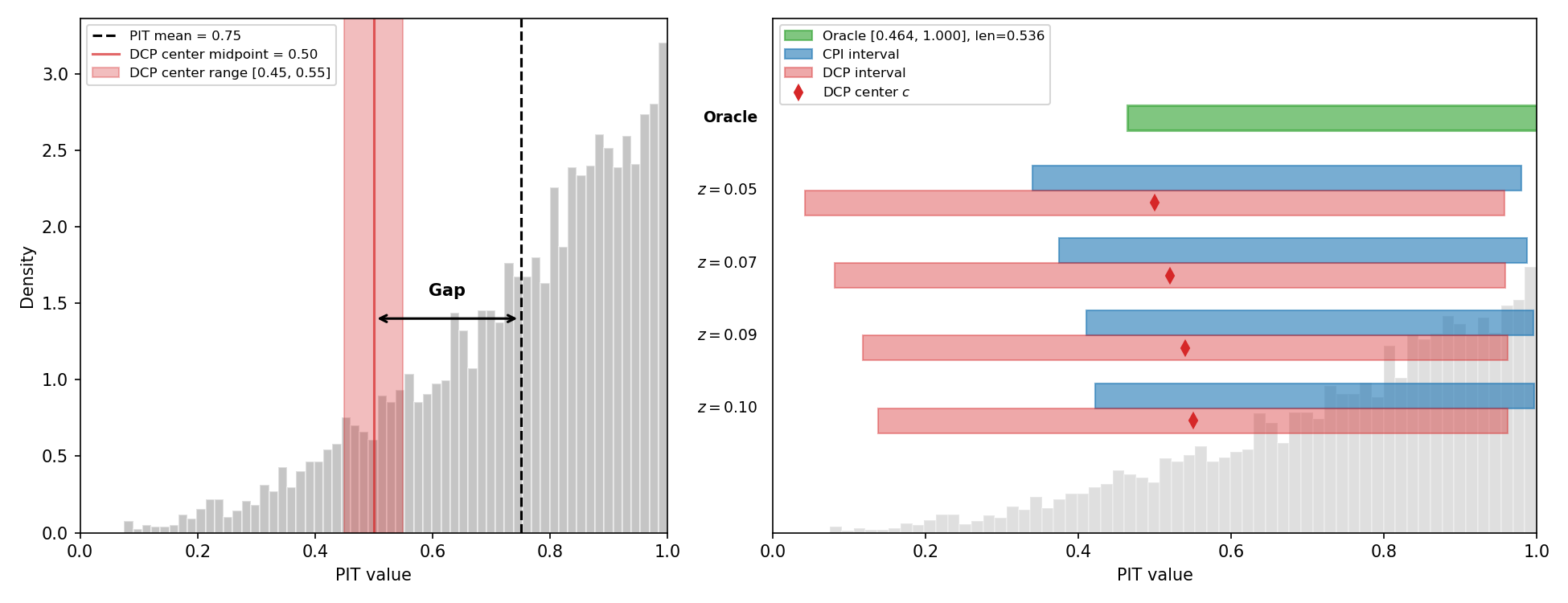}
  \caption{Left: DCP center range $[0.45, 0.55]$ (red shading) versus the mean of $\mathrm{Beta}(3,1)$. Right: the oracle shortest $90\%$ interval (green), together with CPI and DCP interval endpoints from a single calibration draw at four starting points $z$.}
  \label{fig:cutoff_beta31}
\end{figure}

\begin{table}[t]
\centering
\caption{Averaged coverage and PIT interval length for CPI and DCP on $\mathrm{Beta}(3,1)$. The oracle shortest 90\% interval has length 0.536. Length reduction is relative to DCP.}
\label{tab:avg_results}
\begin{tabular}{c cc ccc }
\toprule
& \multicolumn{2}{c}{Coverage} & \multicolumn{3}{c}{Interval length}  \\
\cmidrule(lr){2-3} \cmidrule(lr){4-6}
$z$ & CPI & DCP  & CPI & DCP & Red. \\
\midrule
0.05 & .891\,  & .896\,    & .618\,(.038) & .931\,(.015) & 33.6\% \\
0.06 & .891\,  & .896\,    & .598\,(.037) & .911\,(.015) & 34.3\% \\
0.07 & .893\,  & .896\,    & .581\,(.036) & .891\,(.015) & 34.8\% \\
0.08 & .895\,  & .896\,    & .565\,(.034) & .872\,(.015) & 35.2\% \\
0.09 & .896\,  & .896\,   & .551\,(.034) & .852\,(.015) & 35.4\% \\
0.10 & .892\,  & .896\,    & .536\,(.033) & .833\,(.015) & 35.6\% \\
\bottomrule
\end{tabular}
\end{table}
\textbf{Illustration in PIT space.}
 To illustrate this, consider a right-skewed PIT distribution such as $\mathrm{Beta}(3,1)$. The oracle shortest 90\% interval is $[0.464, 1.000]$ with length $0.536$, concentrating entirely in the high-density region. As shown in Figure~\ref{fig:cutoff_beta31}, CPI's intervals closely track this oracle, while DCP's restricted center forces its lower bound to remain artificially low, preventing it from efficiently capturing the high-density region.
 We quantify this advantage via a simulation with $n_{\mathrm{cal}} = 200$, $\alpha = 0.1$, and 5000 independent replications. Table~\ref{tab:avg_results} reports the averaged coverage and PIT interval length for CPI and DCP on $\mathrm{Beta}(3,1)$ at starting points $z = 0.05, 0.06, \ldots, 0.10$. Both methods maintain the target $90\%$ coverage, but CPI achieves $34\%-36\%$ shorter intervals than DCP across all choices of $z$, and approaches the oracle length 0.536 as $z$ increases.
 
 The shorter PIT intervals (particularly in cases where the DCP interval almost entirely contains the CPI interval) translate to shorter prediction intervals on the response scale. This explains why the CPI intervals are narrower than those of DCP in practice. We further verify in the numerical experiments that this advantage persists when mapping back to the original response ($Y$) space.

\section{Theoretical Guarantees}
\label{sec:theory}
We begin by proving a finite-sample marginal coverage guarantee for the proposed CPI methods, including both the basic approach and the length-optimal version. We then strengthen this result by establishing 
the asymptotic guarantee for conditional coverage.
\begin{theorem}[Finite-Sample Marginal Guarantee]
\label{thm:finite_123}
 If $(X_i, Y_i), i = 1, \dots, n+1$ are exchangeable, then the prediction interval $\widehat C(X_{n+1})$ constructed by the CPI method, as well as its length-optimal version, satisfies  
\[
\mathbb P\!\left(Y_{\mathrm{n+1}}\in \widehat C(X_{\mathrm{n+1}})\ \right)\ \ge\ 1-\alpha.
\]
\end{theorem}

\begin{proof} 
Recall that $U_j=\widehat F(Y_j\mid X_j)$, $j=1,\dots,n$, denote the PIT values  for the calibration data. Similarly, let $$U_{n+1} = \widehat F(Y_{n+1}\mid X_{n+1}),$$ 
and let $R_{n+1}$ denote the rank of $U_{n+1}$ among $U_1, \dots, U_n, U_{n+1}$. 

We provide the proof for an arbitrary function $z(x)$, provided it is strictly independent of the calibration set. This includes, for example, the fixed constant $z = \alpha/2$, or a model  trained exclusively on the training set, as in our length-optimal version. 
Therefore the functions $L(\cdot)$ and $H(\cdot)$ defined in \emph{Step 2}  do not depend on the calibration data. 

Given $X_{n+1}$, the quantities $L(X_{n+1})$ and $H(X_{n+1})$ are fixed, and moreover are independent of the ordering of the PIT values $U_1,\dots,U_{n+1}$. Since the data points $(X_i, Y_i)$ for $i = 1, \dots, n+1$ are exchangeable, and $\widehat{F}$ is trained on a separate dataset (making it independent of these points), it follows that $U_1,\dots,U_{n+1}$ are also exchangeable. Consequently, the rank $R_{n+1}$ is uniformly distributed, even after conditioning on $X_{n+1}$.
Therefore,
\[
\begin{split}
\mathbb{P}\big(L(X_{n+1}) \le R_{n+1} \le H(X_{n+1}) \mid X_{n+1}\big)\\= \frac{H(X_{n+1}) - L(X_{n+1}) + 1}{n+1}.  
\end{split}\]
Taking expectation over $X_{n+1}$, we obtain
\[
\begin{split}
    \mathbb{P}\big(L(X_{n+1}) \le R_{n+1} \le H(X_{n+1})\big)\\= \mathbb{E}\left[\frac{H(X_{n+1}) - L(X_{n+1}) + 1}{n+1}\right].
\end{split}\]
From the CPI construction in \emph{Step 2}, the indices $L(X_{n+1})$ and $H(X_{n+1})$ are chosen such that
$$H(X_{n+1}) - L(X_{n+1}) +1 \ge (1-\alpha)(n+1),$$
for all $X_{n+1}$. Hence,
\[
\begin{split}
    \mathbb P( L(X_{n+1})\le R_{n+1}\le H(X_{n+1})) \ge\ 1-\alpha.
\end{split}
\]
By definition of rank, the event
\[\{ U_{(L(X_{n+1}))}\le  U_{n+1}\le U_{(H(X_{n+1}))}\}\]
is equivalent to
\[\{L(X_{n+1})\le R_{n+1}\le H(X_{n+1})\}.\] 
This implies
\[
\begin{split}
    \mathbb P( U_{(L(X_{n+1}))}\le  U_{n+1}\le  U_{(H(X_{n+1}))} ) \\
    = \mathbb P( L(X_{n+1})\le R_{n+1}\le H(X_{n+1}))\ge\ 1-\alpha.
\end{split}
\]
Since $\widehat F(\cdot\mid x)$ is nondecreasing and $\widehat Q(\cdot\mid x)$ is its generalized inverse, the event $$\{U_{(L(X_{n+1}))}\le  U_{n+1}\le  U_{(H(X_{n+1}))}\}$$ is equivalent to $$\widehat Q(U_{(L(X_{n+1}))}\mid X_{n+1})\le Y_{n+1}\le \widehat Q(U_{(H(X_{n+1}))}\mid X_{n+1}),$$ that is
\[Y_{n+1}\in \widehat C(X_{n+1}).\] 
Therefore, the desired coverage result follows.
\end{proof}

From the proof we see that the finite-sample marginal coverage guarantee follows the same fundamental philosophy as conformal prediction. Coverage is achieved through a rank-based argument that does not rely on distributional assumptions or model correctness. In contrast, achieving conditional coverage is substantially more challenging. Without additional assumptions, exact conditional coverage is generally impossible as shown in prior impossibility results for distribution-free prediction \citep{foygel2021limits, angelopoulos2024theoretical}. 

We next provide an asymptotic conditional validity result under standard assumptions.  This asymptotic analysis follows a theoretical trajectory similar to DCP \cite{chernozhukov2021dcp}, in the sense that conditional validity is driven by consistency of the estimated conditional CDF and the resulting asymptotic uniformity of the PIT representation.

\begin{assumption}[Conditional consistency of $\widehat F$]
\label{ass:cond_consistency}
For almost every $x$, 
$
\sup_{y\in\mathbb R}\big|\widehat F(y\mid x)-F(y\mid x)\big| = o_{P}(1)
$
(as the training sample size tends to infinity).
\end{assumption}

\begin{assumption}[Regularity of $F$]
\label{ass:regularity}
For almost every $x$, the conditional distribution of $Y\mid X=x$ is continuous, and $F(\cdot\mid x)$ is strictly increasing on the support of $Y\mid X=x$.
\end{assumption}

\begin{theorem}[Asymptotic conditional coverage]
\label{thm:asymp_conditional}
Suppose Assumptions~\ref{ass:cond_consistency}--\ref{ass:regularity} hold. Then  the prediction interval $\widehat C(X_{n+1})$ constructed by the CPI method satisfies
\[
\mathbb P\!\left(Y_{n+1}\in \widehat C(X_{n+1})\mid X_{n+1}\right)\ = \ 1-\alpha + o_{P}(1).
\]

\end{theorem} 

\begin{proof}
For notational simplicity, write \((X,Y)\) for \((X_{n+1},Y_{n+1})\) and 
$U = \widehat F(Y\mid X)$.
If \(\widehat F = F\), then by Assumption~\ref{ass:regularity} we have \(U\mid X=x\sim \mathrm{Unif}(0,1)\) for almost every \(x\). Since \(\widehat F\) is estimated, define the conditional PIT error
\[
\Delta(x):=\sup_{u\in[0,1]}\big|\mathbb P(U\le u\mid X=x)-u\big|.
\]
Let \(U_{(L(x))}\) and \(U_{(H(x))}\) be the test-dependent PIT cutoffs computed from the calibration PIT values \(U_1,\dots,U_n\). By the definition of \(\Delta(x)\), we have
\[
\begin{split}
\Big|
\mathbb P\!\left(U_{(L(x))}\le U \le U_{(H(x))}\mid X=x\right)\quad\quad\quad\quad \\
\quad -\,
\mathbb E\!\left[U_{(H(x))}-U_{(L(x))}\mid X=x\right]
\Big|
\le 2\Delta(x).
\end{split}
\]

The calibration PIT values \(U_1,\dots,U_n\) are i.i.d. with marginal CDF \(F_U(u):=\mathbb P(U\le u)\). Moreover, 
\[
\sup_{u\in[0,1]}|F_U(u)-u| 
\le \mathbb E[\Delta(X)]=: \bar\Delta.
\]
Note also that \(\big|(H(x)-L(x))/(n+1)-(1-\alpha)\big|=O(1/n)\). Combining this with standard order-statistic/quantile theory yields
\[
\Big|
\mathbb E\big[U_{(H(x))}-U_{(L(x))}\mid X=x\big] - (1-\alpha)
\Big|
\le 2\bar\Delta + O\!\left(\frac{1}{n}\right).
\]

Putting the last two displays together gives
\[
\Big|
\mathbb P\!\left(Y\in \widehat C(x)\mid X=x\right) - (1-\alpha)
\Big|
\le
2\Delta(x) + 2\bar\Delta + O\!\left(\frac{1}{n}\right).
\]
Under Assumption~\ref{ass:cond_consistency} and Assumption~\ref{ass:regularity}, we have \(\Delta(x)=o_P(1)\) for almost every \(x\) as the training size increases, and consequently \(\bar\Delta=\mathbb E[\Delta(X)]\to 0\). Letting both the training size and \(n\) tend to infinity yields
\[
\mathbb P\!\left(Y\in \widehat C(x)\mid X=x\right)=1-\alpha+o_P(1)
\]
for almost every \(x\), which proves the claim.
\end{proof}

 This theorem establishes the asymptotic conditional validity of our procedure, assuming the consistency of the conditional CDF estimator. The coverage error is driven by how close the estimated PIT distribution is to the uniform distribution. Importantly, this asymptotic validity does not depend on the accuracy of $z_{\theta}$ (or $\hat{z}$), which affects efficiency rather than coverage. If $z_{\theta}$ (or $\hat{z}$) accurately estimates $z^{\star}$, the resulting prediction interval will be close to the oracle optimal interval $\big[F^{-1}(z^{\star}(x)\mid x), F^{-1}(z^{\star}(x)+1-\alpha \mid x)\big]$. 
 
 Finally, we emphasize that verifying these limits is not a prerequisite for the practical validity of our method; rather, the theorem serves as a theoretical sanity check confirming that, under ideal conditions (i.e., consistent CDF estimation and large calibration samples), the method recovers conditional validity. The primary strength of CPI lies in its finite-sample marginal guarantees and its ability to produce shorter, better-adapted intervals in practice.

\section{Simulations}
\label{sec:simulation}
In this section, we provide a comprehensive empirical evaluation of the proposed CPI framework through a series of simulation studies. First, we assess the performance of CPI against four existing conformal prediction methods including DCP, CQR, and both standard residual (referred to as Residual) and variance-rescaled residual (referred to as Rescaled) approaches under complex data-generating processes characterized by multimodality and covariate-dependent noise. This comparison highlights the advantages of our method in terms of conditional coverage and interval efficiency. Second, we move beyond the i.i.d. setting to evaluate the robustness of the two CDF-based methods (CPI and DCP) under distribution shift, demonstrating how our adaptive calibration mechanism maintains efficiency even when the underlying CDF estimates are systematically biased.


\subsection{Conditional Coverage and Efficiency Analysis}
\label{subsec:sim_dgp}
\textbf{Setup.} We consider a covariate vector $x \in \mathbb{R}^{5}$ whose coordinates are  independent and distributed as
\begin{align*}
x_{1} &\sim \mathrm{Uniform}(0,1), \\
x_{2} &\sim \mathcal{N}(0,1) \ \text{truncated to } [-3,3], \\
x_{3} &\sim \mathrm{Beta}(0.5,\,0.5), \\
x_{4} &\sim \mathrm{Bernoulli}(0.5), \\
x_{5} &\sim \mathrm{Poisson}(2) \ \text{truncated at } 5.
\end{align*}
The response is generated according to
\begin{equation*}
    y \;=\; x_{1}^{2} + x_{2} x_{3} + x_{3} x_{4} + x_{5} + \varepsilon,
\end{equation*}
where the noise term is conditionally heteroscedastic in  $x_{1}$:
$$\varepsilon \;=\; g(x_{1}) \sum_{k=1}^{3} \pi_{k}(x_{1})\,\mathcal{D}_{k}$$  with scale
$g(x_{1}) = 0.05 + 1.5\,(x_{1} - 0.5)^{2}$, and independent mixture components $\mathcal{D}_{1}\sim t_{3}$, $\mathcal{D}_{2}\sim\mathcal{N}(0,1)$,  $\mathcal{D}_{3}\sim\mathrm{Exp}(1)-1$. The mixture weights are defined as
\begin{equation*}
    \pi_{1}(x_{1}) = 4(x_{1}-0.5)^{2}\,\mathbbm{1}\{x_{1}<0.5\},
    \quad
    \pi_{3}(x_{1}) = 4(x_{1}-0.5)^{2}\,\mathbbm{1}\{x_{1}\geq 0.5\},
    \quad
    \pi_{2}(x_{1}) = 1 - \pi_{1}(x_{1}) - \pi_{3}(x_{1}).
\end{equation*}
This construction yields $\mathbb{E}[\varepsilon \mid x_{1}] = 0$, while the conditional scale, skewness, and tail behavior of $\varepsilon$ all vary smoothly with $x_{1}$, transitioning from heavy-tailed near $x_{1}=0$ to right-skewed near $x_{1}=1$.

During pretraining, $n_{\mathrm{tr}}=2000$ samples are used to train estimators. Both CPI and DCP are based on the same NN-CDE estimator \citep{hu2023nncde} and utilize Approach 1 to select PIT starting point $z$. Details of the estimator, as well as those of the other approaches, are provided in the Appendix. 

The calibration set size is set to $n_{\mathrm{cal}}=300$ with a target error rate of $\alpha=0.1$. To assess conditional coverage, we fix $200$ $x_1$ covariate values sampled uniformly from $(0,1)$ and generate their corresponding test samples. All results are averaged over $1000$ independent replications.

\textbf{Marginal coverage and interval length.} Table~\ref{tab:sim_2b_cov_width_vertical} summarizes the marginal coverage and average interval width across all test covariates. As expected, all methods achieve coverage levels near the nominal target. Notably, CPI produces the shortest prediction intervals, demonstrating superior efficiency while maintaining valid coverage.

\begin{table}[t]
\centering
\caption{Marginal coverage and interval width }
\label{tab:sim_2b_cov_width_vertical}
\small
\begin{tabular}{lcc}
\toprule
Method & Coverage & Width \\
\midrule
CPI  & 0.906 (0.065) & 0.601 (0.334) \\
DCP      & 0.910 (0.065) & 0.616 (0.363) \\
Residual            & 0.910 (0.140) & 0.784 (0.000) \\
Rescaled            & 0.910 (0.132) & 0.792 (0.119) \\
CQR                 & 0.899 (0.148) & 0.915 (0.227) \\
\bottomrule
\end{tabular}
\end{table}

\textbf{Conditional coverage.} 
Figure~\ref{fig:sim_2b_binned_cov} shows the conditional coverage across different values of $x_1$, together with the nominal target level of $0.9$ (dashed) and an under-coverage warning threshold of $0.8$ (dotted). For visualization, the smoothed curves are obtained using Gaussian kernel (Nadaraya--Watson) smoothing, with bandwidth chosen by Silverman's rule. Although all methods attain the target marginal coverage, their conditional profiles differ sharply. The non--CDF-based methods (Residual, Rescaled, and CQR) are highly unstable: they over-cover in the interior, where the noise variance is small, and drop well below the $0.8$ warning threshold at both boundaries, where the noise variance is large (the variation of noise is illustrated in Figure~\ref{fig:sim_2b_noise}). In contrast, the CDF-based methods (CPI  and DCP) track the nominal level closely across the entire range of $x_1$, with coverage remaining above the warning threshold at every $x_1$.

\textbf{Conditional interval width.} Next, we examine efficiency, measured 
by the width of the prediction intervals. 
Figure~\ref{fig:sim_2b_binned_width}   illustrates how interval width 
varies with $x_1$, using the same smoothing procedure as above. While the 
Residual, Rescaled, and CQR methods fail to adapt to the noise structure shown in Figure~\ref{fig:sim_2b_noise} and Figure~\ref{fig:sim_2b_condvar}, 
producing either constant or counter-intuitive width patterns, both 
CDF-based methods (CPI and DCP) exhibit U-shaped curves that closely track 
local variability. While CPI and DCP perform similarly over the 
central range of $x_1$, CPI becomes more efficient near the boundaries, especially as $x_1 \to 1$, 
producing intervals approximately 5.8\% shorter in the $(0.8, 1]$ bin. This 
suggests that in regimes with complex noise---characterized by heavy tails 
or skewness---CPI achieves the target coverage with tighter intervals than DCP. As we demonstrate in the next simulation, this 
efficiency gap becomes even more critical under distribution shift.

\begin{figure}[H]
    \centering
    \begin{subfigure}[h]{0.48\linewidth}
        \centering
        \includegraphics[width=\linewidth]{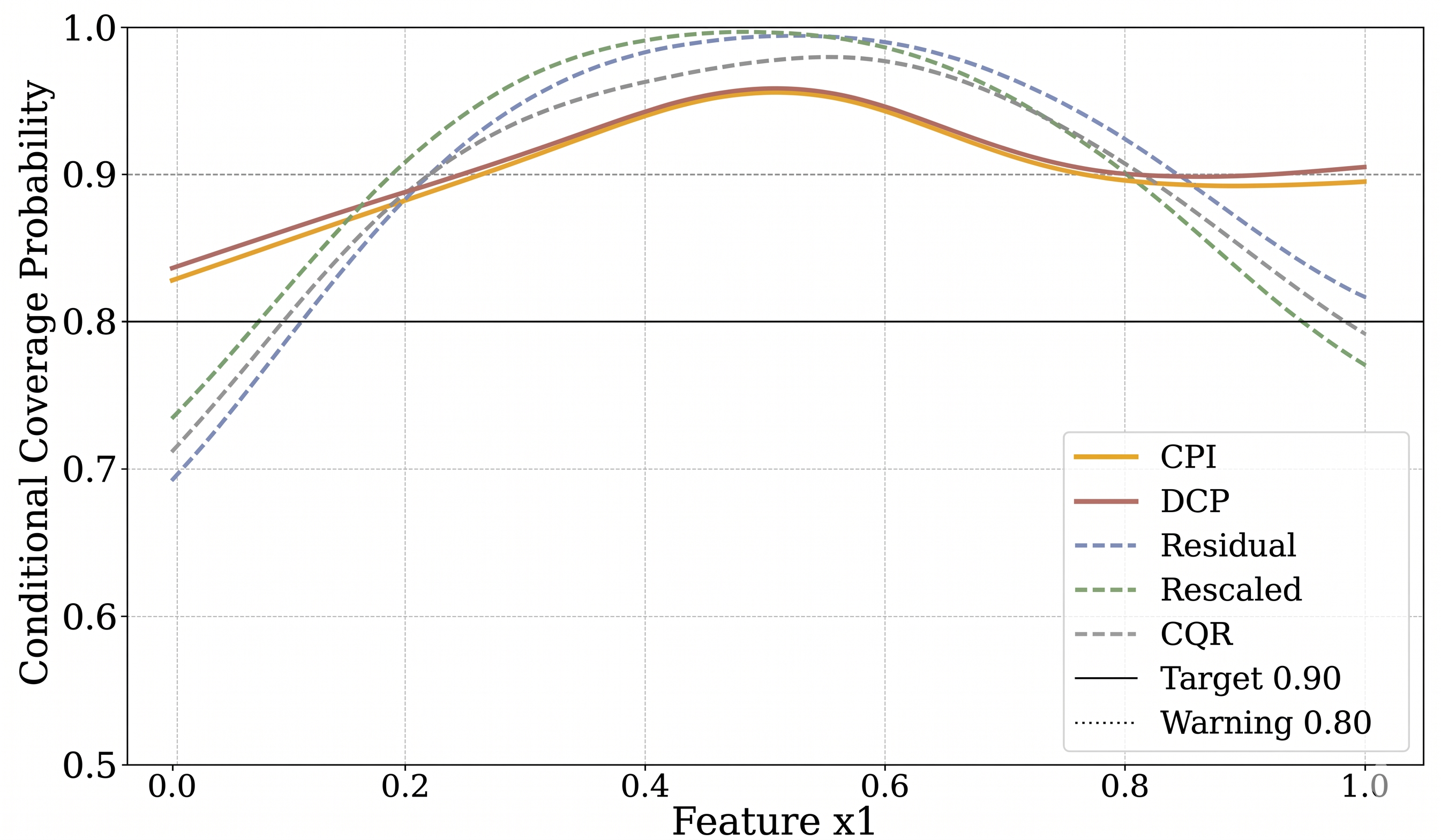}
        \caption{Smoothed conditional coverage. Dashed line: nominal 
        level $0.9$; dotted line: under-coverage threshold $0.8$.}
        \label{fig:sim_2b_binned_cov}
    \end{subfigure}
    \hfill
    \begin{subfigure}[h]{0.48\linewidth}
        \centering
        \includegraphics[width=\linewidth]{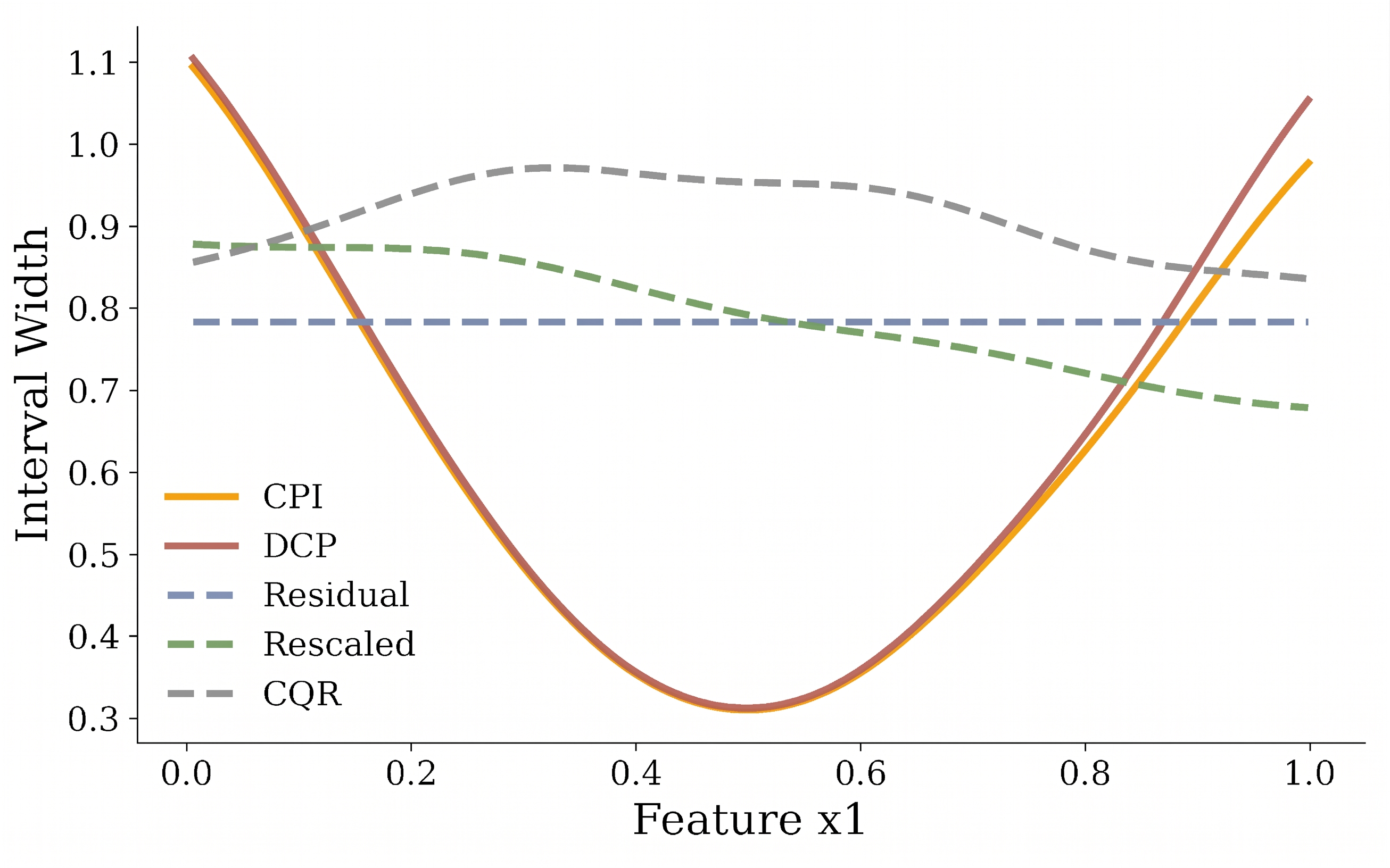}
        \caption{Smoothed conditional interval width.}
        \label{fig:sim_2b_binned_width}
    \end{subfigure}

    \vspace{0.3cm}

    \begin{subfigure}[h]{0.48\linewidth}
        \centering
        \includegraphics[width=\linewidth]{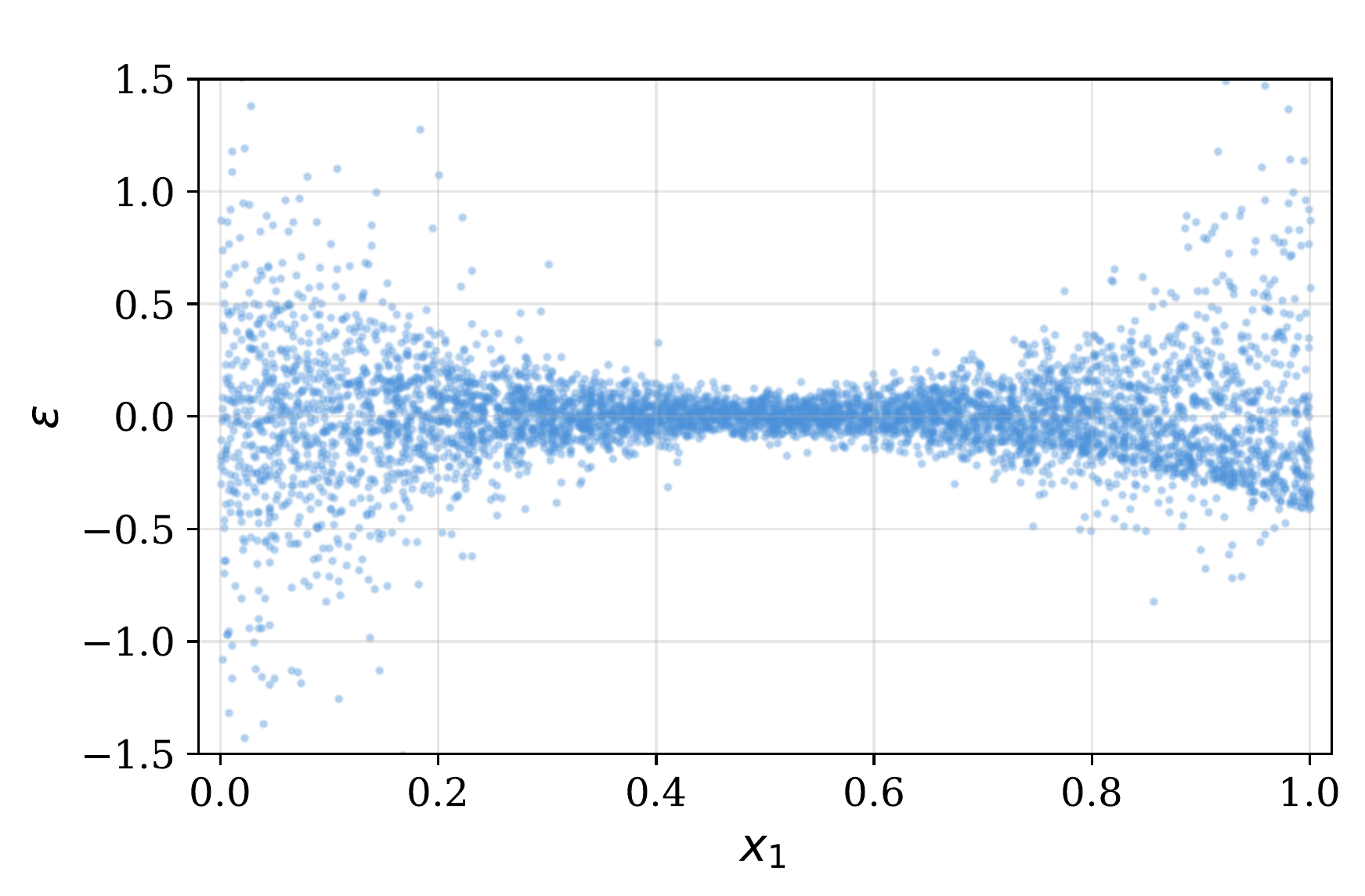}
        \caption{Scatter plot of Noise $\epsilon$ against $x_1$, showing heteroscedastic variability.}
        \label{fig:sim_2b_noise}
    \end{subfigure}
    \hfill
    \begin{subfigure}[h]{0.48\linewidth}
        \centering
        \includegraphics[width=\linewidth]{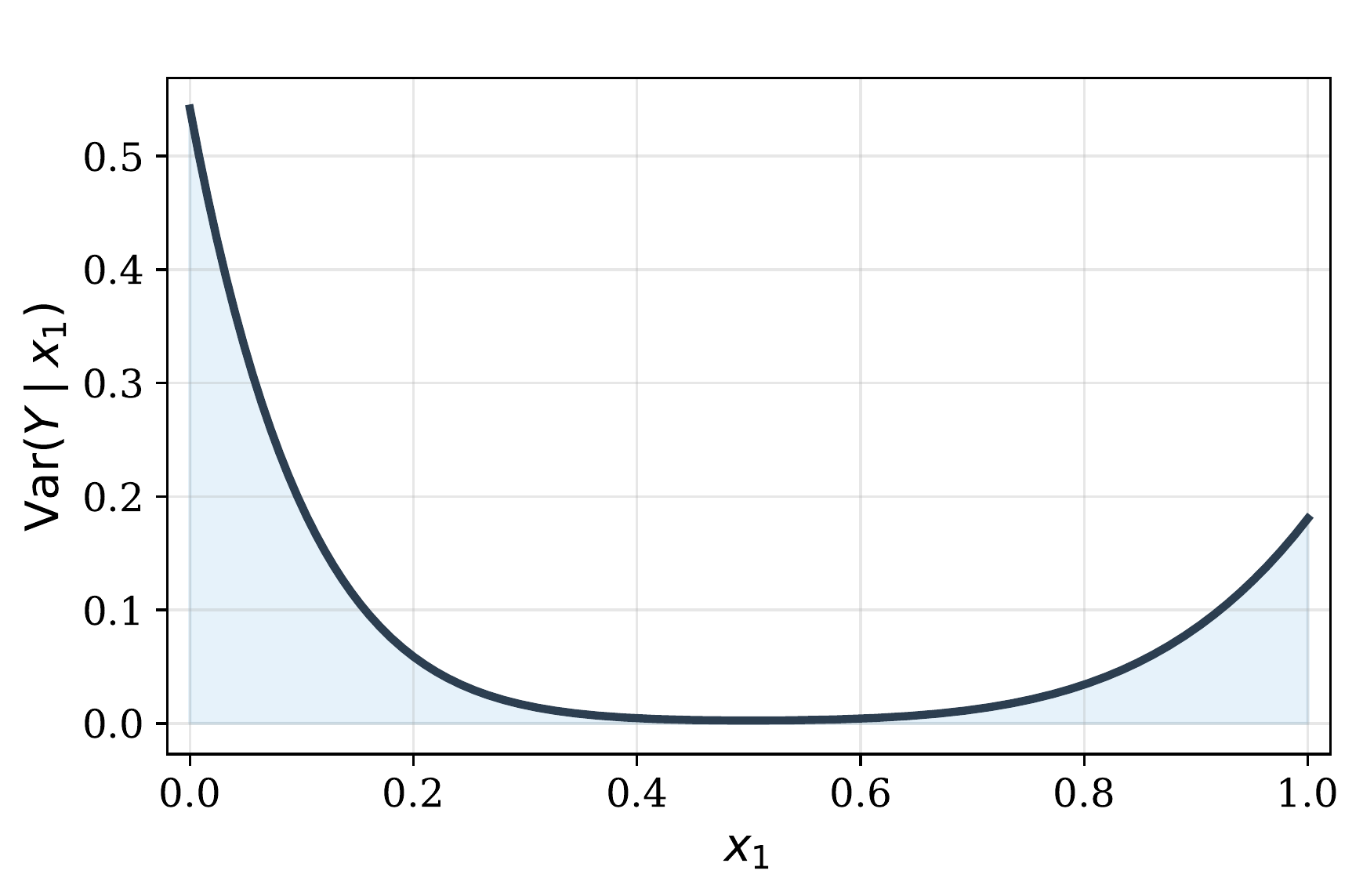}
        \caption{Theoretical $\mathrm{Var}(Y \mid x_1)$.  }
        \label{fig:sim_2b_condvar}
    \end{subfigure}
    \caption{conditional performance and noise structure.}
    \label{fig:sim_2b_cov_width}
\end{figure}



\subsection{Comparison between CPI and DCP  under 
Distribution Shift}
\label{sec:sim_shift_cpi_dcp} 
Previous simulations have assumed that calibration and test data originate from the same distribution as the training data. In practice, however, this assumption is often violated, for instance, by evolving noise structures over time. A robust conformal procedure must therefore retain valid coverage while maintaining stable efficiency.   In this section, we probe this robustness through a controlled distribution-shift experiment, demonstrating how our adaptive calibration mechanism allows CPI to achieve significantly tighter, more stable intervals than DCP when the estimated CDF is systematically shifted.

\paragraph{Setup.}
We reuse the data-generating process in Section~\ref{subsec:sim_dgp} and introduce a location shift on 
the response,
\begin{equation}
    Y' \;=\; f(X) \;+\; g(X_1)\,\varepsilon \;+\; \delta,
    \qquad
    \delta \in \{0,\ 0.05,\ 0.1\},
    \label{eq:shift}
\end{equation}
applied \emph{only} to the calibration and test samples; the CDF estimator (NN-CDE 
model) is trained on the unshifted distribution 
($\delta = 0$) and then held fixed. This design isolates the effect 
of the shift on the conformal calibration step, so that any 
difference between CPI and DCP can be attributed to 
their calibration mechanisms rather than to the underlying 
conditional density estimator. 

\begin{table}[t]
    \centering
    \small
    \caption{Conditional and marginal coverage and  interval 
    length under the location shift. Standard deviations across replications are in 
    parentheses. ``Red.'' denotes the relative reduction in mean 
    interval length of CPI over DCP, i.e., 
    $1 - \mathrm{len}(\text{CPI}) / \mathrm{len}(\text{DCP})$.}
    \label{tab:shift_0c1}
    \begin{tabular}{cc cc ccc}
        \toprule
        & & \multicolumn{2}{c}{Coverage} 
          & \multicolumn{3}{c}{Interval length} \\
        \cmidrule(lr){3-4} \cmidrule(lr){5-7}
        $\delta$ & $x_1$ bin & CPI & DCP 
            & CPI & DCP & Red. \\
        \midrule
        0.00 & $(0.0,0.2]$ & 0.837 & 0.845 & 1.033 (0.259) & 1.043 (0.260) & 1.0\% \\
             & $(0.2,0.4]$ & 0.906 & 0.909 & 0.464 (0.173) & 0.468 (0.176) & 0.9\% \\
             & $(0.4,0.6]$ & 0.973 & 0.976 & 0.282 (0.048) & 0.283 (0.049) & 0.6\% \\
             & $(0.6,0.8]$ & 0.906 & 0.908 & 0.442 (0.105) & 0.449 (0.108) & 1.5\% \\
             & $(0.8,1.0]$ & 0.894 & 0.902 & 0.894 (0.270) & 0.949 (0.369) & 5.8\% \\
             & \textbf{Marginal} & 0.906 & 0.910 & 0.601 (0.334) & 0.616 (0.363) & 2.5\% \\
        \midrule
        0.05 & $(0.0,0.2]$ & 0.866 & 0.880 & 1.116 (0.263) & 1.157 (0.256) & 3.6\% \\
             & $(0.2,0.4]$ & 0.914 & 0.916 & 0.497 (0.181) & 0.533 (0.181) & 6.6\% \\
             & $(0.4,0.6]$ & 0.970 & 0.959 & 0.297 (0.047) & 0.328 (0.047) & 9.4\% \\
             & $(0.6,0.8]$ & 0.902 & 0.904 & 0.425 (0.108) & 0.454 (0.132) & 6.3\% \\
             & $(0.8,1.0]$ & 0.871 & 0.858 & 0.802 (0.216) & 1.315 (1.059) & 39.0\% \\
             & \textbf{Marginal} & 0.906 & 0.904 & 0.603 (0.329) & 0.733 (0.630) & 17.7\% \\
        \midrule
        0.10 & $(0.0,0.2]$ & 0.870 & 0.908 & 1.190 (0.303) & 1.505 (0.692) & 20.9\% \\
             & $(0.2,0.4]$ & 0.914 & 0.908 & 0.513 (0.200) & 0.617 (0.211) & 16.8\% \\
             & $(0.4,0.6]$ & 0.948 & 0.914 & 0.298 (0.051) & 0.382 (0.053) & 22.0\% \\
             & $(0.6,0.8]$ & 0.916 & 0.903 & 0.457 (0.134) & 1.935 (6.398) & 76.4\% \\
             & $(0.8,1.0]$ & 0.903 & 0.877 & 0.919 (0.264) & 25.526 (24.892) & 96.4\% \\
             & \textbf{Marginal} & 0.912 & 0.902 & 0.649 (0.373) & 5.965 (15.040) & 89.1\% \\
        \bottomrule
    \end{tabular}
\end{table}

\paragraph{Results.}
Table~\ref{tab:shift_0c1} reports coverage and interval length across different shift levels ($\delta\in\{0, 0.05, 0.1\}$).
Across all regimes, both methods successfully maintain near-nominal marginal coverage, illustrating the effectiveness of the conformal framework in providing finite-sample valid coverage through model-free calibration, even when the underlying density estimator is degraded by distribution shift. However, a significant performance gap emerges regarding efficiency. Under a mild shift ($\delta=0.05$), CPI achieves a 17.7\% reduction in marginal mean width, with a 39.0\% reduction in the boundary bin $x_1 \in (0.8, 1.0]$. This efficiency gap widens dramatically under a larger shift ($\delta=0.1$). In this regime, DCP becomes highly unstable; its marginal mean width reaches 5.965—nearly ten times the 0.649 value achieved by CPI. This performance degradation is even more pronounced in the bin $x_1 \in (0.8, 1.0]$,  where DCP  produces a mean width of 25.526 with a massive standard deviation of 24.892, indicating extreme sensitivity to the shift. In  contrast, CPI remains robust, maintaining a mean width of only 0.919 with a standard deviation of 0.264.

\section{Real Data}
\label{sec:real_data}

In this section, we evaluate the methods on six real datasets: Airfoil~\citep{airfoil}, Computer~\citep{computer}, Abalone~\citep{abalone}, Concrete~\citep{concrete}, AutoMPG~\citep{autompg}, and Crime~\citep{crime}. Each dataset is split into three disjoint subsets: $45\%$ for training, $35\%$ for calibration, and $20\%$ for testing, and all results are averaged over $10$ different random partitions.
We compare CPI against the same four approaches as in the simulations, namely DCP, CQR, Residual, and Rescaled. All methods are calibrated at the nominal level $1-\alpha = 0.90$. 

\paragraph{Marginal coverage and interval width.}
Tables~\ref{tab:real-data-coverage} and~\ref{tab:real-data-width} report the marginal coverage and average interval width on the six datasets. All methods achieve coverage close to the nominal 90\% level on every dataset, confirming the validity of the conformal calibration. Under this comparable coverage, CPI consistently produces the shortest prediction intervals among all methods. 

\begin{table}[ht]
\centering
\caption{Marginal coverage on six real-world datasets. All methods achieve coverage close to the nominal 90\% level.}
\label{tab:real-data-coverage}
\begin{tabular}{lccccc}
\toprule
Dataset & Residual & Rescaled & CQR & CPI & DCP \\
\midrule
Airfoil   & 0.899 & 0.913 & 0.910 & 0.902 & 0.907 \\
Computer  & 0.905 & 0.903 & 0.897 & 0.897 & 0.901 \\
Abalone   & 0.895 & 0.896 & 0.898 & 0.892 & 0.893 \\
Concrete  & 0.884 & 0.912 & 0.897 & 0.900 & 0.910 \\
AutoMPG   & 0.915 & 0.888 & 0.894 & 0.904 & 0.926 \\
Crime     & 0.904 & 0.901 & 0.907 & 0.908 & 0.912 \\
\bottomrule
\end{tabular}
\end{table}

\begin{table}[ht]
\centering
\caption{Average interval width on six real-world datasets. Bold indicates the best (shortest) width for each dataset. CPI achieves the narrowest intervals on all six datasets.}
\label{tab:real-data-width}
\begin{tabular}{lccccc}
\toprule
Dataset & Residual & Rescaled & CQR & CPI & DCP \\
\midrule
Airfoil   & 15.34 & 12.71 & 12.42 & \textbf{11.03} & 12.31 \\
Computer  & 21.75 & 25.63 & 10.62 & \textbf{8.68} & 8.71 \\
Abalone   & 8.61 & 8.29 & 8.52 & \textbf{7.96} & 8.12 \\
Concrete  & 34.50 & 34.92 & 31.35 & \textbf{20.82} & 22.57 \\
AutoMPG   & 19.77 & 22.37 & 20.21 & \textbf{18.38} & 20.60 \\
Crime     & 0.72 & 0.94 & 0.81 & \textbf{0.70} & 1.61 \\
\bottomrule
\end{tabular}
\end{table}

\paragraph{Conditional coverage.}To assess conditional coverage, we project the test features onto the first principal component (PC1), computed from the standardized training data, and partition the test set into four groups according to the quartiles of the projected values. Within each group, we compute the group-specific coverage and mean interval width for each method. CPI consistently achieves the most stable coverage close to the nominal $90\%$ level, while yielding the shortest intervals within different subgroups across various datasets. The other four approaches either suffer from significant undercoverage in certain groups or exhibit excessively large interval widths.
Figure~\ref{fig:pc1-abalone} shows the results on the Abalone dataset. Across the four PC1 groups, CPI and DCP achieve coverage close to the nominal 90\% level. In contrast, Rescaled slightly undercovers in Group 2, while Residual and CQR fall below the target level in Group 3. Moreover, CPI consistently achieves the shortest intervals across all four groups.
These results demonstrate that CPI not only yields shorter intervals marginally but also maintains well-calibrated conditional coverage across different regions of the feature space. Analogous plots for the remaining five datasets are provided in Appendix~\ref{sec:appendix_condcov}, where the same pattern holds consistently.

\begin{figure}[H]
\centering
\includegraphics[width=0.9\textwidth]{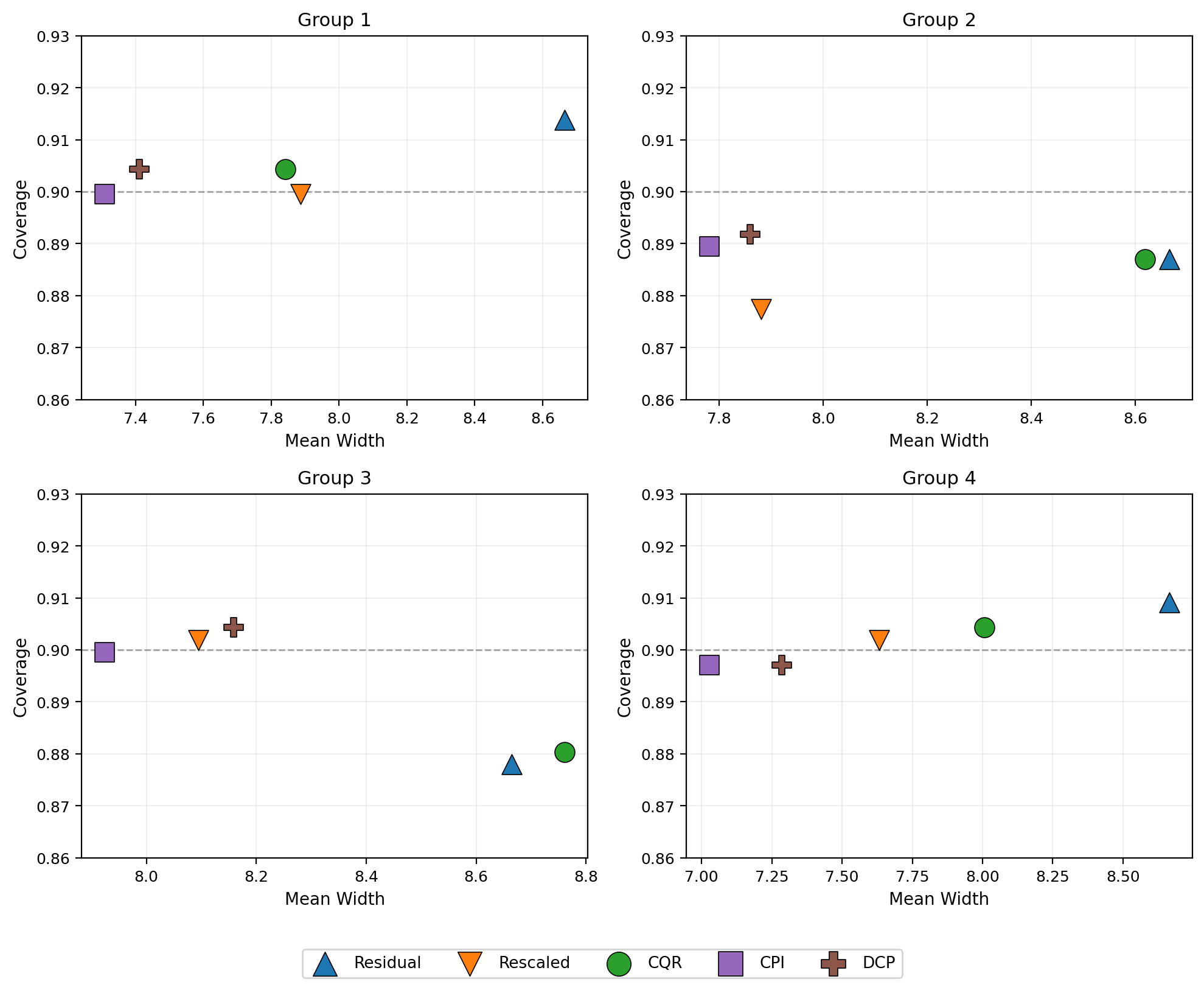}
\caption{Coverage vs interval width on the Abalone dataset, stratified by PC1 quartile groups.}
\label{fig:pc1-abalone}
\end{figure}

\section{Discussion}
In this work, we introduced a new framework for constructing prediction intervals based on
conditional CDF estimation combined with a finite-sample calibration step. Our method follows the
spirit of conformal prediction in using held-out data to guarantee marginal coverage, while departing from standard approaches by calibrating percentile endpoints directly rather than relying on a single conformity score. By leveraging the estimated CDF, our approach improves conditional adaptivity relative to non-CDF-based methods. At the same time, it advances existing CDF-based conformal inference by making the calibration step robust to asymmetric errors in the estimated CDF—errors that naturally arise from finite-sample bias, model misspecification, or distribution shift.

The method is simple to implement and can be paired with any conditional CDF estimator. Empirically, CPI yields more stable conditional coverage across the feature space and produces consistently shorter prediction intervals. These gains arise from better alignment with the underlying conditional distribution, rather than a trade-off between coverage and interval length.

Overall, our results suggest that directly calibrating percentile endpoints provides a flexible and effective alternative to score-based conformal methods, particularly in settings with heteroscedasticity, asymmetry, or distributional complexity.

\bibliography{example_paper}
\bibliographystyle{icml2026}  


\appendix
\renewcommand{\thesection}{\Alph{section}}
\setcounter{section}{0}

\begin{center}
  {\Large\bfseries Appendix}
\end{center}
\addcontentsline{toc}{section}{Appendix}
\vspace{0.5em}

\section{Implementation Details}
\label{sec:impl_details}

We provide the full implementation details for reproducibility.

\paragraph{NN-CDE (Hazard Network).}
The conditional density estimator is a two-hidden-layer MLP with $64$ hidden units per layer and ReLU activations. All features are standardized to zero mean and unit variance. Training uses Adam~\citep{kingma2015adam} with learning rate $5 \times 10^{-4}$, batch size $256$, up to $500$ epochs with early stopping (patience $40$) on a $20\%$ validation split.

\paragraph{Optimal $z^*(\mathbf{x})$ computation.}
The shortest-width starting quantile $z^*(\mathbf{x})$ is found by grid search over $41$ equally spaced values in $(10^{-6},\; \alpha - 10^{-6})$. Quantile inversion uses linear interpolation on the cumulative hazard grid.

\paragraph{Baseline methods (simulation).}
All baseline neural networks are trained with Adam (learning rate $10^{-3}$, batch size $256$, up to $180$ epochs, early stopping with patience $10$, $30\%$ validation split).

\begin{itemize}
    \item \textbf{Residual}: a two-hidden-layer MLP (\texttt{SimpleNN}) with $32$ hidden units per layer and ReLU activations, trained with MSE loss. 
    \item \textbf{Rescaled}: a two-hidden-layer shared backbone (\texttt{MeanVarianceNN}, $32$ hidden units) with separate linear heads for the mean $\hat{\mu}(\mathbf{x})$ and log-variance $\log\hat{\sigma}^2(\mathbf{x})$, trained with Gaussian negative log-likelihood loss. The conformal score is the standardized residual $|y - \hat{\mu}(\mathbf{x})| / \hat{\sigma}(\mathbf{x})$.
    \item \textbf{CQR}: a two-hidden-layer MLP (\texttt{QuantileNN}) with $16$ hidden units, outputting two quantiles simultaneously, trained with pinball loss at levels $\alpha/2$ and $1 - \alpha/2$.
\end{itemize}

\paragraph{Real data experiments.}
Each dataset is split $45\%/35\%/20\%$ (train/calibration/test), and all results are averaged $10$ random partitions. Our implementation follows the simulation setup with one exception: Rescaled and CQR employ a two-hidden-layer MLP with $48$ hidden units. Training is conducted using Adam with a learning rate of $10^{-3}$ and a batch size of $128$, for up to $200$ epochs with an early stopping patience of $20$. The NN-CDE architecture and all calibration procedures remain identical.

\section{Conditional Coverage on Real Datasets}
\label{sec:appendix_condcov}

This appendix provides the per-group coverage vs.\ mean interval width plots for the five real datasets not shown in the main text. For each dataset, test points are projected onto the first principal component (PC1) of the standardized training features and partitioned into four groups by PC1 quartiles. Each scatter plot shows the group-specific coverage (vertical axis) and mean interval width (horizontal axis) for the five conformal methods, averaged over $10$ random partitions. The dashed line indicates the nominal $90\%$ coverage level.

Across all datasets, CPI consistently achieves coverage close to the nominal level. The sole exception is the Computer dataset, where all approaches exhibit undercoverage in Group 4. Notably, CPI maintains the narrowest prediction intervals in each PC1 group, confirming that its marginal interval width advantage extends to the conditional setting. 

\begin{figure}[H]
\centering
\includegraphics[width=\textwidth]{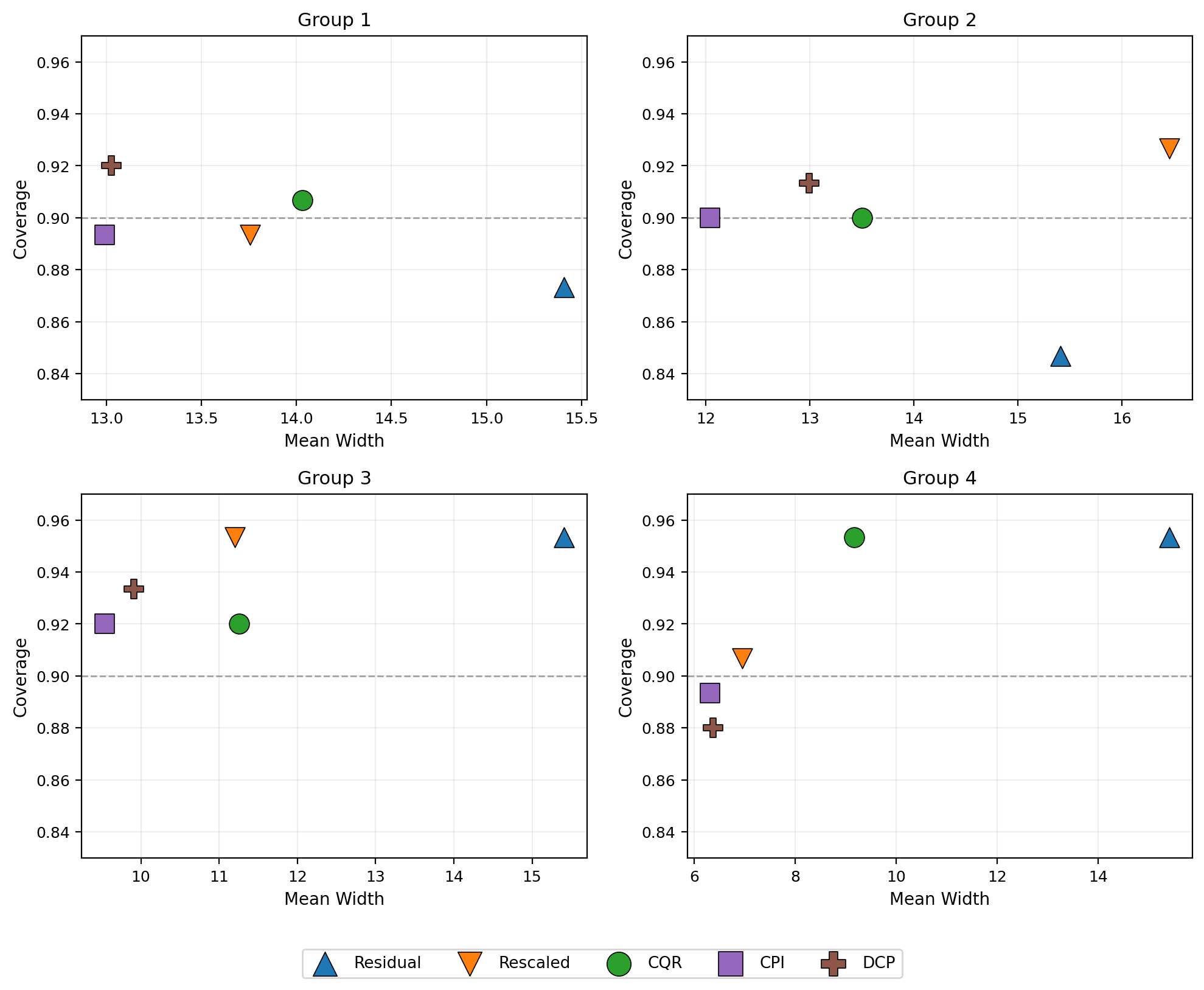}
\caption{Coverage vs.\ mean interval width on the Airfoil dataset, stratified by PC1 quartile groups.}
\label{fig:pc1-airfoil}
\end{figure}

\begin{figure}[H]
\centering
\includegraphics[width=\textwidth]{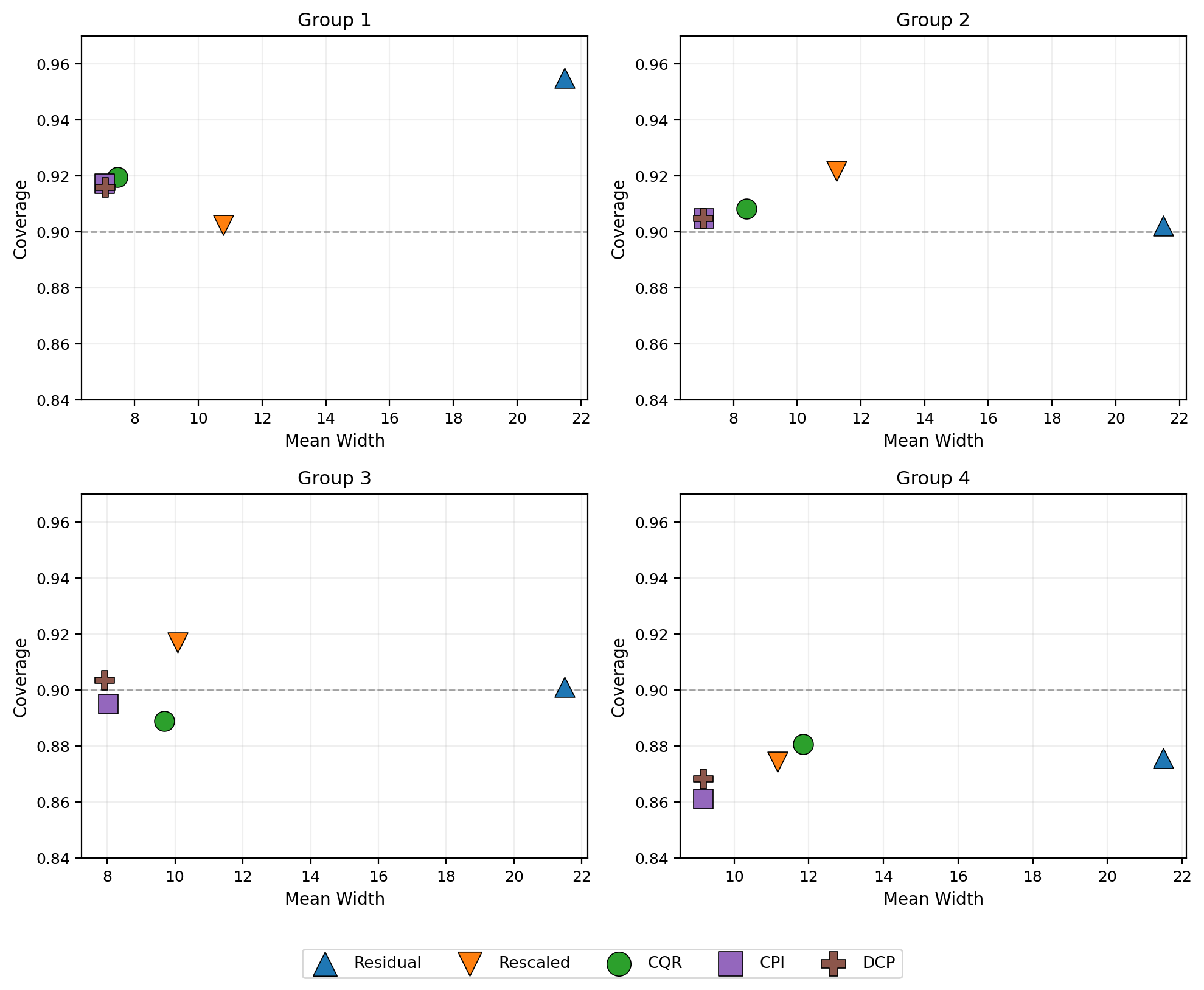}
\caption{Coverage vs.\ mean interval width on the Computer dataset, stratified by PC1 quartile groups.}
\label{fig:pc1-computer}
\end{figure}

\begin{figure}[H]
\centering
\includegraphics[width=\textwidth]{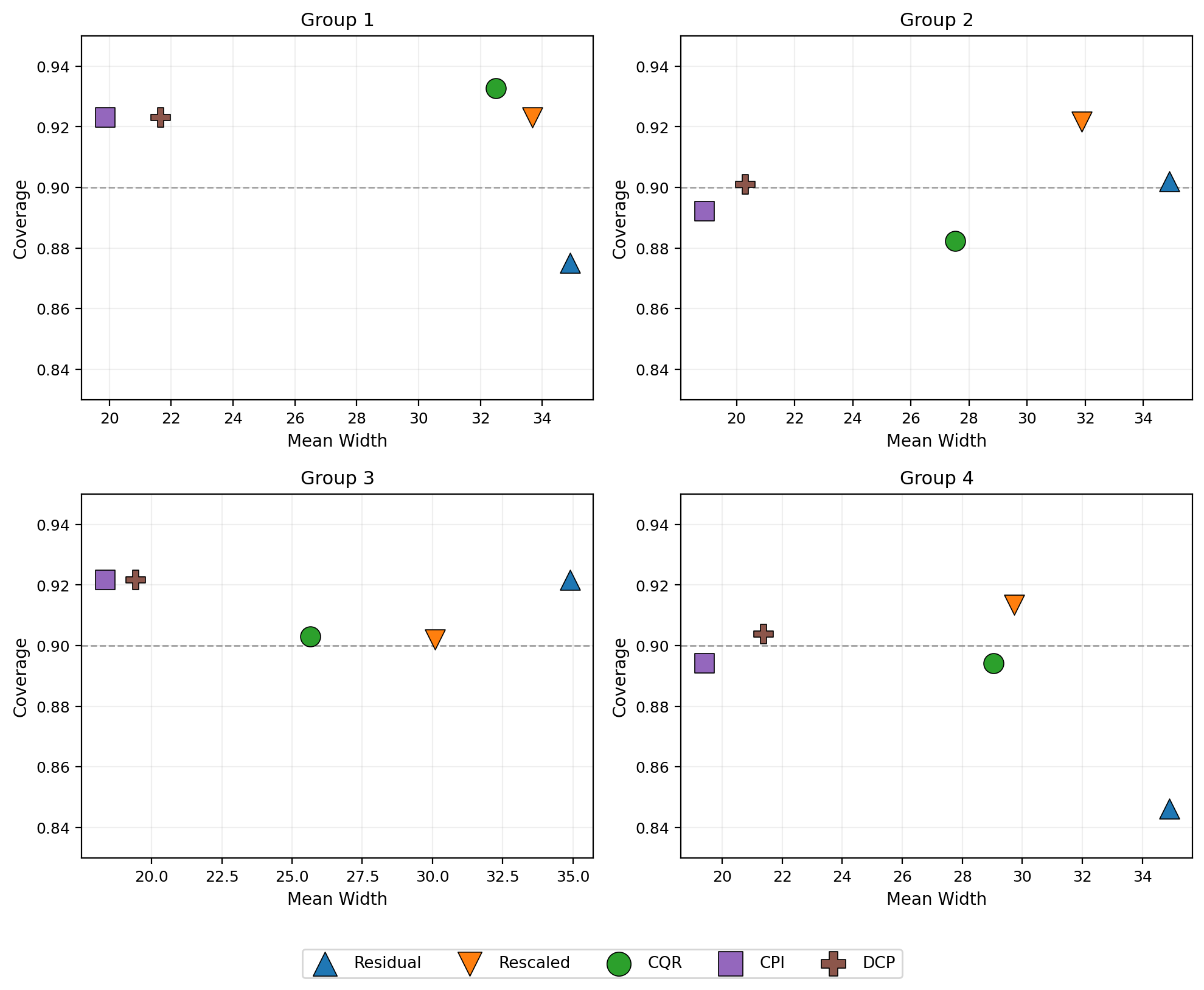}
\caption{Coverage vs.\ mean interval width on the Concrete dataset, stratified by PC1 quartile groups.}
\label{fig:pc1-concrete}
\end{figure}

\begin{figure}[H]
\centering
\includegraphics[width=\textwidth]{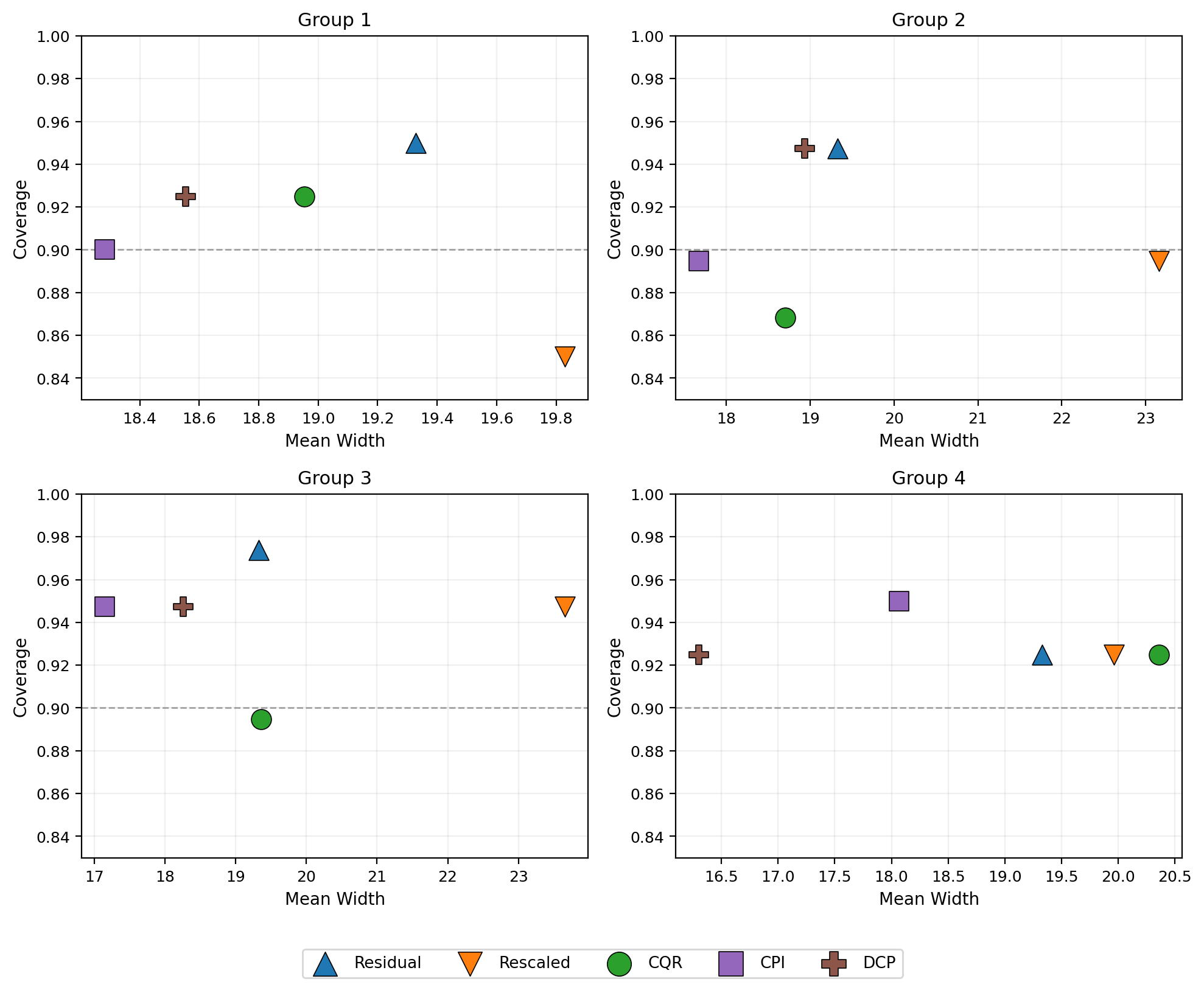}
\caption{Coverage vs.\ mean interval width on the AutoMPG dataset, stratified by PC1 quartile groups.}
\label{fig:pc1-autompg}
\end{figure}

\begin{figure}[H]
\centering
\includegraphics[width=\textwidth]{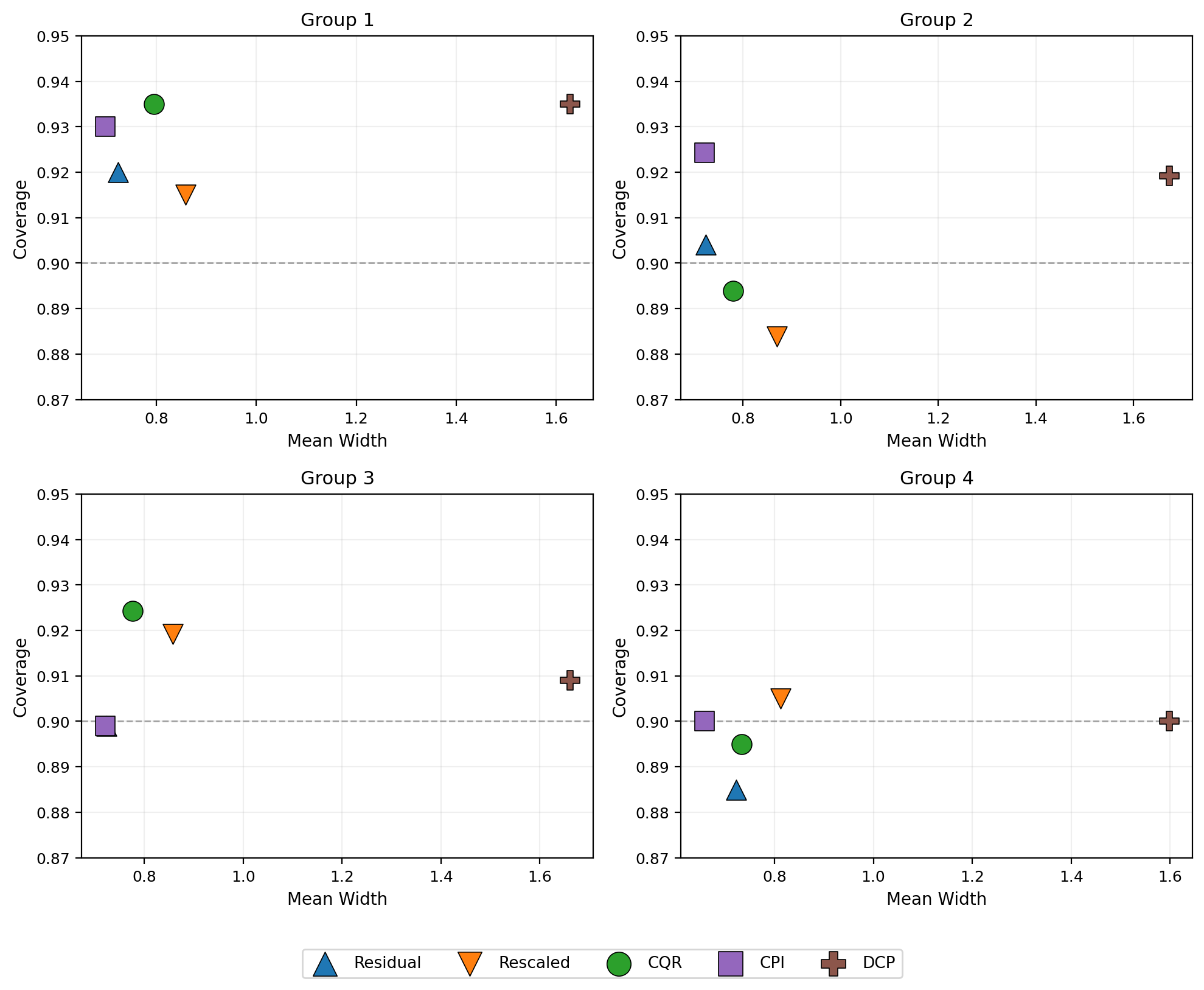}
\caption{Coverage vs.\ mean interval width on the Crime dataset, stratified by PC1 quartile groups.}
\label{fig:pc1-crime}
\end{figure}

\end{document}